\documentclass[11pt]{article}

\usepackage{fullpage}
\usepackage{amsmath}
\usepackage{amsthm}
\usepackage{amssymb}
\usepackage{amsfonts}
\usepackage{hyperref}
\usepackage{color}
\usepackage{xcolor}

\usepackage{graphicx} %package to manage images
\graphicspath{ {./images/} }
\usepackage[rightcaption]{sidecap}
\usepackage{wrapfig}

\usepackage{enumitem}
\setlist{itemsep=-3pt}

\newcommand{\ignore}[1]{}

\newtheorem{theorem}{Theorem}

\newtheorem{lemma}{Lemma}

\renewcommand{\Pr}{{\bf Pr}}
\newcommand{\E}{{\bf E}}

\newcommand{\D}{{\cal D}}

\newcommand{\SSS}{{\cal S}}

\newcommand{\GF}{{\rm {GF}}}

\newcommand{\wt}{{\rm wt}}

\renewcommand{\P}{\mathbb{P}}

\newcommand{\mybox}[1]{\noindent\fbox{\parbox{\textwidth}{#1}}}
\newcounter{counter}
\newcommand{\correct}[1]{\addtocounter{counter}{1}\color{green}$$\mbox{{\Huge Block \arabic{counter}}}$$
\color{red}$$\mbox{{\huge (1) Correctness (2) Consistency}}$$
$$\mbox{{\Large (3) Google Scholor (4) Yellow (5) Grammarly pdf}}$$\color{blue} }

\begin{document}

\title{Almost Optimal Proper Learning and Testing Polynomials}
\author{{\bf Nader H. Bshouty}\\ Dept. of Computer Science\\ Technion, Haifa, Israel.\\
}

\maketitle
\begin{abstract}

We give the first almost optimal polynomial-time proper learning algorithm of Boolean sparse multivariate polynomial under the uniform distribution. For $s$-sparse polynomial over $n$ variables and $\epsilon=1/s^\beta$, $\beta>1$, our algorithm makes $$q_U=\left(\frac{s}{\epsilon}\right)^{\frac{\log \beta}{\beta}+O(\frac{1}{\beta})}+ \tilde O\left(s\right)\left(\log\frac{1}{\epsilon}\right)\log n$$ queries. Notice that our query complexity is sublinear in $1/\epsilon$ and almost linear in $s$. All previous algorithms have query complexity at least quadratic in $s$ and linear in $1/\epsilon$.

We then prove the almost tight lower bound
$$q_L=\left(\frac{s}{\epsilon}\right)^{\frac{\log \beta}{\beta}+\Omega(\frac{1}{\beta})}+ \Omega\left(s\right)\left(\log\frac{1}{\epsilon}\right)\log n,$$ 

Applying the reduction in~\cite{Bshouty19b} with the above algorithm, we give the first almost optimal polynomial-time  tester for $s$-sparse polynomial. Our tester, for $\beta>3.404$, makes
$$\tilde O\left(\frac{s}{\epsilon}\right)$$ queries.
\end{abstract}
\section{Introduction}

In this paper, we study the learnability and testability of the class of sparse (multivariate) polynomials over $\GF(2)$. A polynomial over $\GF(2)$ is the sum in $\GF(2)$ of monomials, where a monomial is a product of variables. It is well known that every Boolean function has a unique representation as a (multilinear) polynomial over $\GF(2)$. A Boolean function is called $s$-sparse polynomial if its unique polynomial expression contains at most $s$ monomials. 

\subsection{Learning}
In the learning model~\cite{Ang88,Valiant84}, the learning algorithm has access to a black-box query oracle to a function $f$ that is $s$-sparse polynomial. The goal is to run in $poly(n,s,1/\epsilon)$ time, make $poly(n,s,1/\epsilon)$ black-box queries and, with probability at least $2/3$, learn a Boolean function $h$ that is $\epsilon$-close to $f$ under the uniform distribution, i.e., $\Pr_x[f(x)\not=h(x)]\le \epsilon$. The learning algorithm is called {\it proper learning} if it outputs an $s$-sparse polynomial. The learning algorithm is called {\it exact learning algorithm} if $\epsilon=0$. 

Proper and non-proper learning algorithms of $s$-sparse polynomials that run in polynomial-time and make a polynomial number of queries have been studied by many authors~\cite{BeimelBBKV00,BergadanoBV96,BishtBM06,BlumS90,Bshouty97,Bshouty19b,BshoutyM02,ClausenDGK91,DurG93,FischerS92,HellersteinS07,RothB91,SchapireS96}. 

For learning $s$-sparse polynomial without black-box queries (PAC-learning without black-box queries, \cite{Valiant84}) and for exact learning ($\epsilon=0$) the following results are known. In \cite{HellersteinS07}, Hellerstein and Servedio gave a non-proper learning algorithm that learns only from random examples under any distribution  that runs in time $n^{O(n\log s)^{1/2}}$. Roth and Benedek, \cite{ClausenDGK91}, show that for any $s\ge 2$ polynomial-time proper PAC-learning without black-box queries of $s$-sparse polynomials implies RP=NP.
They gave a proper exact learning ($\epsilon=0$) algorithm that makes $(n/\log s)^{\log s}$ black-box queries. They also show that to exactly learn $s$-sparse polynomial, we need at least $(n/\log s)^{\log s}$ black-box queries. See also~\cite{ClausenDGK91}.

For polynomial-time non-proper and proper learning $s$-sparse polynomial with black-box queries under the uniform distribution, all the algorithms in the literature, \cite{BeimelBBKV00,BergadanoBV96,BishtBM06,BlumS90,Bshouty97,Bshouty19b,BshoutyM02,ClausenDGK91,FischerS92,RothB91,SchapireS96}, have query complexities that are at least quadratic in $s$ and linear in $1/\epsilon$.  
In this paper, we prove

\begin{theorem}\label{one} Let $\epsilon=1/s^\beta$.
There is a proper learning algorithm for $s$-sparse polynomial that runs in polynomial-time and makes $$q_U=\left(\frac{s}{\epsilon}\right)^{\frac{\log \beta}{\beta}+O(\frac{1}{\beta})}+\tilde O\left(s\right)\left(\log\frac{1}{\epsilon}\right)\log n$$ queries. 
\end{theorem}
To the best of our knowledge, this is the first learning algorithm whose query complexity is sublinear in $1/\epsilon$ and almost linear in $s$. 

We then give the following lower bound that shows that our query complexity is almost optimal.

\begin{theorem}\label{two} Let $\epsilon=1/s^\beta$.
Any learning algorithm for $s$-sparse polynomial must make at least $$q_L=\left(\frac{s}{\epsilon}\right)^{\frac{\log \beta}{\beta}+\Omega(\frac{1}{\beta})}+\Omega\left(s\left(\log\frac{1}{\epsilon}\right)\log n\right)$$ queries. 
\end{theorem}

\subsection{Property Testing}
A problem closely related to learning polynomial is the problem of property testing polynomial: Given black-box query access to a Boolean function $f$. Distinguish, with high probability, the case that $f$ is $s$-sparse polynomial versus the case that $f$ is $\epsilon$-far from every $s$-sparse polynomial. Property testing of Boolean function was first considered in the seminal works of Blum, Luby and Rubinfeld~\cite{BlumLR93} and Rubinfeld and Sudan~\cite{RubinfeldS96} and has recently become a very active research area. See the surveys and books~\cite{GoldreichSurvey10,Goldreich17,Ron08,Ron09}.

In the uniform distribution framework, where the distance between two functions is measured with respect to the uniform distribution, the first testing algorithm for $s$-sparse polynomial runs in exponential time~\cite{DiakonikolasLMORSW07} and makes $\tilde O(s^4/\epsilon^2)$ queries. Chakraborty et al. \cite{ChakrabortyGM11}, gave another exponential time algorithm that makes $\tilde O(s/\epsilon^2)$ queries. Diakonikolas et al. gave in~\cite{DiakonikolasLMSW11} the first polynomial-time testing algorithm that makes $poly(s,1/\epsilon)>s^{10}/\epsilon^3$ queries. In \cite{Bshouty19b}, Bshouty gave a polynomial-time  algorithm that makes $\tilde O(s^2/\epsilon)$ queries. As for the lower bound for the query complexity, the lower bound $\Omega(1/\epsilon)$ follows from Bshouty and Goldriech lower bound in~\cite{BshoutyG22}. Blaise et al. \cite{BlaisBM11}, and Saglam,~\cite{Saglam18}, gave the lower bound $\Omega(s\log s)$.

In this paper, we show 

\begin{theorem}\label{three} For any $\epsilon=1/s^\beta$ there is an algorithm for $\epsilon$-testing $s$-sparse polynomial that makes $$Q=\left(\frac{s}{\epsilon}\right)^{\frac{\log \beta}{\beta}+O(\frac{1}{\beta})}+\tilde O\left(\frac{s}{\epsilon}\right)$$ queries.

In particular, for $\beta>3.404$,
$$Q= \tilde O\left(\frac{s}{\epsilon}\right).$$
\end{theorem}

Notice that the query complexity of the tester in Theorem~\ref{three} is $\tilde O((1/\epsilon)^{1+1/\beta})$. This is within a factor of $(1/\epsilon)^{1/\beta}$ of the lower bound $\Omega(1/\epsilon)$. Therefore, the query complexity in Theorem~\ref{three} is almost optimal. 

\section{Techniques}\label{sec2}
In this section, we give a brief overview of the techniques used for the main results, Theorem~\ref{one},~\ref{two}, and~\ref{three}.
\subsection{Upper Bound}
This section gives a brief overview of the proof of Theorem~\ref{one}.

Our algorithm first reduces the learning of $s$-sparse polynomial to exact learning $s$-sparse polynomials with monomials of size at most $d=O(\log(s/\epsilon))$, i.e., degree-$d$ $s$-sparse polynomials. Given an $s$-sparse polynomial $f$, we project each variable to $0$ with probability $O(\log s/\log(1/\epsilon))$. In this projection, monomials of size greater than $\Omega(d)$ vanish, with high probability. Then we learn the projected function. We take enough random zero projections of $f$ so that, with high probability, for every monomial $M$ of $f$ of size at most $\log(s/\epsilon)$, there is a projection $q$ such that $M$ does not vanish under~$q$. Collecting all the monomials of degree at most $\log(s/\epsilon)$ in all the projections gives a hypothesis that is $\epsilon$-close to the target function $f$.

Now to exactly learn the degree-$d$ $s$-sparse polynomials, where $d=O(\log(s/\epsilon))$, we first give an algorithm that finds a monomial of a degree-$d$ $s$-sparse polynomial that makes
\begin{eqnarray}\label{key}
Q=2^{dH_2\left(\frac{ \log s}{d}\right)(1-o_s(1))}\log n = s^{\log(\log(s/\epsilon)/\log s)+O(1)}\log n
\end{eqnarray} queries where $H_2$ is the binary entropy. The best-known algorithm for this problem has query complexity $Q'=2^d\log n\approx poly(s/\epsilon)\log n$, \cite{Bshouty19b,BshoutyM02}. For small enough $\epsilon$, $Q'\gg Q$. The previous algorithm in~\cite{Bshouty19b} chooses uniformly at random assignments until it finds a positive assignment $a$, i.e., $f(a)=1$. Then recursively do the same for $f(a*x)$, where $a*x=(a_1x_1,a_2x_2,\ldots,a_nx_n)$, until no more $a$ with smaller Hamming weight can be found. Then $f(a*x)=\prod_{a_i=1}x_i$ is a monomial of~$f$. To find a positive assignment in a degree $d$ polynomial from uniformly at random assignments, we need to make, on average, $2^d$ queries. The number of nonzero entries in $a*x$ is on average $n/2$. Therefore, this algorithm makes $O(2^d\log n)$ queries. In this paper, we study the probability $\Pr_{\D_p}[f(a)=1]$ when~$a$ is chosen according to the product distribution $\D_p$, where each $a_i$ is equal to~$1$ with probability $p$ and is $0$ with probability $1-p$. We show that to maximize this probability, we need to choose $p=1-(\log s)/d$. Replacing the uniform distribution with the distribution $\D_p$ in the above algorithm gives the query complexity in~(\ref{key}). 

Now, let $f$ be a degree-$d$ $s$-sparse polynomial, and suppose we have learned some monomials $M_1,\ldots,M_t$ of $f$. To learn a new monomial of $f$, we learn a monomial of $f+h$ where $h=M_1+M_2+\cdots+M_t$. 
This gives an algorithm that makes,
\begin{eqnarray}\label{key2}
q=\left(\frac{s}{\epsilon}\right)^{\frac{\log \beta}{\beta}+O(\frac{1}{\beta})}\log n
\end{eqnarray} queries where $\epsilon=1/s^\beta$. All previous algorithms have query complexity that are at least quadratic in $s$ and linear in $1/\epsilon$.

Now, notice that the query complexity in (\ref{key2}) is not the query complexity that is stated in Theorem~\ref{one}. To get the query complexity in the theorem, we use another reduction. This reduction is from exact learning degree-$d$ $s$-sparse polynomials over $n$ variables to exact learning degree-$d$ $s$-sparse polynomials over $m=O(d^2s^2)$ variables. Given a degree-$d$ $s$-sparse polynomials $f$ over $n$ variables. We choose uniformly at random a projection $\phi:[n]\to [m]$ and learn the polynomial $F(x_1,\ldots,x_m)=f(x_{\phi(1)},\ldots,x_{\phi(n)})$ over $m$ variables. This is equivalent to distributing the $n$ variables, uniformly at random, into $m$ boxes, assigning different variables for different boxes, and then learning the function with the new variables. We choose $m=O(d^2s^2)$ so that different variables in $f$ fall into different boxes. By~(\ref{key2}), the query complexity of learning $F$ is
\begin{eqnarray}\label{RR1}
q'=\left(\frac{s}{\epsilon}\right)^{\frac{\log \beta}{\beta}+O(\frac{1}{\beta})}\log m=\left(\frac{s}{\epsilon}\right)^{\frac{\log \beta}{\beta}+O(\frac{1}{\beta})}.
\end{eqnarray}

After we learn $F$, we find the relevant variables of $F$, i.e., the variables that $F$ depends on. Then, for each relevant variable of $F$, we search for the relevant variable of $f$ that corresponds to this variable. Each search makes $O(\log n)$ queries. The number of relevant variables of $f$ is at most $ds$ and here $d=O(\log(s/\epsilon))$, which adds 
$$\tilde O(s)\left(\log\frac{1}{\epsilon}\right)(\log n)$$ to the query complexity in~(\ref{RR1}). This gives the query complexity in Theorem~\ref{one}. We also show that all the above can be done in time $O(qn)$ where $q$ is the query complexity. 

See more details in Section~\ref{TheAlg}.

\subsection{Lower Bound}
This section gives a brief overview of Theorem~\ref{two}.

In this paper, we give two lower bounds. One that proves the right summand of the lower bound
\begin{eqnarray}\label{lefts}
\Omega\left(s\left(\log\frac{1}{\epsilon}\right)\log n\right),
\end{eqnarray}
and the second proves the left summand 
\begin{eqnarray}\label{rights}
\left(\frac{s}{\epsilon}\right)^{\frac{\log \beta}{\beta}+\Omega(\frac{1}{\beta})}.
\end{eqnarray}
To prove (\ref{lefts}), we consider the class of $\log(1/(2\epsilon))$-degree $s$-sparse polynomials. We show that any learning algorithm for this class can be modified to an exact learning algorithm. Then, using Yao's minimax principle, the query complexity of exactly learning this class is at least $\log$ of the class size. This gives the first lower bound in~(\ref{lefts}).

To prove (\ref{rights}), we consider the class
$$C=\left\{\left.\prod_{i\in I}x_i\prod_{j\in J}(1+x_j) \right| I,J\subseteq [n], |J|\le \log s, |I|\le \log(1/\epsilon)-\log s-1 \right\}.$$ It is easy to see that every polynomial in $C$ is a $s$-sparse polynomial.

Again, we show that any learning algorithm for this class can be modified to an exact learning algorithm. We then use Yao's minimax principle to show that, to exactly learn $C$, we need at least $\Omega(|C|)$ queries. This gives the lower bound in (\ref{rights}).

\subsection{Upper Bound for Testing}

For the result in testing, we use the reduction in~\cite{Bshouty19b}. In~\cite{Bshouty19b} it is shown that given a learning algorithm for $s$-sparse polynomial that makes $q(s,n)$ queries, one can construct a testing algorithm for $s$-sparse polynomial that makes 
$$q(s,\tilde O(s))+\tilde O\left(\frac{s}{\epsilon}\right)$$ queries. 
Using Theorem~\ref{one} we get a testing algorithm with query complexity 
$$\left(\frac{s}{\epsilon}\right)^{\frac{\log \beta}{\beta}+O(\frac{1}{\beta})}+\tilde O\left(\frac{s}{\epsilon}\right).$$ We then show that for $\beta\ge 6.219$, this query complexity is $\tilde O(s/\epsilon)$. In Section~\ref{Smallb}, we give another learning algorithm that has query complexity better than Theorem~\ref{one} for $\beta<6.219$. Using this algorithm, we get a tester that has query complexity $\tilde O(s/\epsilon)$ for $\beta\ge 3.404.$

\section{Definitions and Preliminary Results}\label{definitions}
In this section, we give some definitions and preliminary results.

We will denote by $\P_{n,s}$ the class of $s$-sparse polynomials over the Boolean variables $(x_1,\ldots,x_n)$ and $\P_{n,d,s}\subset \P_{n,s}$, the class of degree-$d$ $s$-sparse polynomials. Formally, let $\SSS_{n,\le d}=\cup_{i\le d}\SSS_{n,i}$, where $\SSS_{n,i}={[n]\choose i}$ is the set of all $i$-subsets of $[n]=\{1,2,\ldots,n\}$. The class $\P_{n,d,s}$ is the class of all the polynomials of the form 
$$\sum_{I\in S}\prod_{i\in I}x_i $$
where $S\subseteq \SSS_{n,\le d}$ and $|S|=s$. The class $\P_{n,s}$ is $\P_{n,n,s}$.

Let $B_n$ be the uniform distribution over $\{0,1\}^n$. The following result is well known. See for example~\cite{BishtBM06}. 
\begin{lemma}\label{weightofp}
For any $f\in \P_{n,d,s}$ we have $\Pr_{x\in B_n}[f(x)=1]\ge 2^{-d}$.
\end{lemma}
We will now extend Lemma~\ref{weightofp} to other distributions.

Let $W_{n,s}$ be the set of all the assignments in $\{0,1\}^n$ of Hamming weight at least $n-\lfloor \log s\rfloor$. We prove the following for completeness~\cite{RothB91,ClausenDGK91}. 
\begin{lemma}\label{Poass}
For any $0\not=f\in \P_{n,s}$ there is an assignment $a\in W_{n,s}$ such that $f(a)=1$.
\end{lemma}
\begin{proof}
We prove the result by induction on $s$. For $s=1$, $f$ is one monomial and therefore $f(1,1,\ldots,1)=1$. Suppose that the statement is true for every $s'\le s-1$. We prove it for $s$. Let $0\not=f\in \P_{n,s}$. Since $s>1$ there is a variable, say wlog $x_1$, such that $f=x_1f_1+f_2$ where $f_1,f_2\not=0$. Consider $f_1=f(0,x_2,\ldots,x_n)$ and $f_2=f(0,x_2,\ldots,x_n)+f(1,x_2,\ldots,x_n)$. One of them has at most $\lfloor s/2\rfloor$ monomials. If $f_1$ does, then by the induction hypothesis there is $a\in W_{n,\lfloor s/2\rfloor}$ such that $f_1(a)=1$. Then for $b=(0,a_2,\ldots,a_n)\in W_{n,s}$, $f(b)=f_1(a)=1$ . If $f_2$ does, then by the induction hypothesis there is $a\in W_{n,\lfloor s/2\rfloor}$ such that $f_2(a)=f_1(a)+f(1,a_1,\ldots,a_n)=1$. Then either $f_1(a)=1$ or $f(1,a_1,\ldots,a_n)=1$ and as before the result holds for $f$.
\end{proof}

The $p$-{\it product distribution} $\D_{n,p}$ is a distribution over $\{0,1\}^n$ where $\D_{n,p}(a)=p^{\wt(a)}(1-p)^{n-\wt(a)}$ where $\wt(a)$ is the Hamming weight of $a$. Let $H_2(x)=-x\log_2x-(1-x)\log_2(1-x)$ be the binary entropy function. 

We prove

\begin{lemma}\label{zerotest} Let $p\ge 1/2$.
For every $f\in \P_{n,d,s}$, $f\not=0$, we have
$$\Pr_{x\in \D_{n,p}}[f(x)=1]\ge \left\{
\begin{array}{ll}
p^{d-\lfloor\log s\rfloor}(1-p)^{ \lfloor\log s\rfloor}& d\ge \lfloor \log s\rfloor\\
(1-p)^d & d< \lfloor \log s\rfloor
\end{array}
.
\right.$$

In particular, if  $d\ge 2\lfloor\log s\rfloor$, then for $p'=(d-\lfloor\log s\rfloor)/d$ 
$$\max_{p\ge 1/2} \Pr_{x\in \D_{n,p}}[f(x)=1]=\Pr_{x\in \D_{n,p'}}[f(x)=1]\ge 2^{-H_2\left(\frac{\lfloor \log s\rfloor}{d}\right)d}$$
and if  $d< 2\lfloor\log s\rfloor$, then for $p'=1/2$ 
$$\max_{p\ge 1/2} \Pr_{x\in \D_{n,p}}[f(x)=1]=\Pr_{x\in \D_{n,p'}}[f(x)=1]\ge 2^{-d}.$$
\end{lemma}
\begin{proof} We first consider the case $n=d$. Let $0\not=f(x)\in \P_{d,d,s}$. When $d\ge \lfloor\log s\rfloor$, by Lemma~\ref{Poass}, there is $a\in W_{d,s}$ such that $f(a)=1$. Therefore  
\begin{eqnarray}\label{fhj}
\Pr_{x\in \D_{d,p}}[f(x)=1]\ge \D_{d,p}(a)=p^{\wt(a)}(1-p)^{d-\wt(a)}\ge p^{d-\lfloor \log s\rfloor}(1-p)^{\lfloor \log s\rfloor}.
\end{eqnarray}
When $d<\lfloor \log s\rfloor$, we have $\P_{d,d,s}=\P_{d,d,2^d}$. This is because $2^d<s$ and any polynomial in $d$ variables of degree $d$ has at most $2^d$ monomials. Therefore, by (\ref{fhj}), we have
\begin{eqnarray*}
\Pr_{x\in \D_{d,p}}[f(x)=1]\ge p^{d-\lfloor \log 2^d\rfloor}(1-p)^{\lfloor \log 2^d\rfloor}=(1-p)^d.
\end{eqnarray*}
Thus, the result follows for nonzero functions with $d$ variables.

 Let $0\not=f(x)\in \P_{n,d,s}$. Let $M$ be a monomial of $f$ of maximal degree $d'\le d$. Assume wlog that $M=x_1x_2\cdots x_{d'}$. First notice that for any $(a_{d+1},\ldots,a_{n})\in \{0,1\}^{n-d}$ we have $g(x_1,\ldots,x_d):=f(x_1,\ldots,x_d,a_{d+1},\ldots,a_n)\not=0$ and $g\in \P_{d,d,s}$.
Consider the indicator random variable $X(x)$ that is equal to $1$ if $f(x)=1$. Then
\begin{eqnarray*}
\Pr_{x\in \D_{n,p}}[f(x)=1]&=&\E[X]\\
&=&\E_{(a_{d+1},\ldots,a_n)\in\D_{n-d,p}}[\E_{(x_1,x_2,\ldots,x_d)\in \D_{d,p}}[X(x_1,\ldots,x_d,a_{d+1},\ldots,a_n)]]\\
&=&\E_{(a_{d+1},\ldots,a_n)\in\D_{n-d,p}}[\Pr_{(x_1,x_2,\ldots,x_d)\in \D_{d,p}}[f(x_1,\ldots,x_d,a_{d+1},\ldots,a_n)=1]]\\
&\ge&\left\{\begin{array}{ll}
p^{d-\lfloor\log s\rfloor}(1-p)^{ \lfloor\log s\rfloor}& d\ge \lfloor \log s\rfloor\\
(1-p)^d & d< \lfloor \log s\rfloor
\end{array}.\right.
\end{eqnarray*}
\end{proof}

In particular, since $f(x)\not=g(x)$ is equivalent to $f(x)+g(x)=1$ and $f+g\in\P_{n,2s,d}$, we have

\begin{lemma}\label{notequal} Let $p\ge 1/2$.
For every $f,g\in \P_{n,d,s}$, $f\not=g$, we have
$$\Pr_{x\in \D_{n,p}}[f(x)\not=g(x)]\ge \left\{
\begin{array}{ll}
p^{d-\lfloor\log s\rfloor-1}(1-p)^{ \lfloor\log s\rfloor+1}& d\ge \lfloor \log s\rfloor+1\\
(1-p)^d & d< \lfloor \log s\rfloor+1
\end{array}
.
\right.$$

In particular, if  $d\ge 2\lfloor\log s\rfloor+2$, then for $p'=(d-\lfloor\log s\rfloor-1)/d$ 
$$\max_{p\ge 1/2} \Pr_{x\in \D_{n,p}}[f(x)\not=g(x)]=\Pr_{x\in \D_{n,p'}}[f(x)\not= g(x)]\ge 2^{-H_2\left(\frac{\lfloor \log s\rfloor+1}{d}\right)d}$$
and if  $d< 2\lfloor\log s\rfloor+2$, then for $p'=1/2$ 
$$\max_{p\ge 1/2} \Pr_{x\in \D_{n,p}}[f(x)\not=g(x)]=\Pr_{x\in \D_{n,p'}}[f(x)\not=g(x)]\ge 2^{-d}.$$

In particular, for $p'=\max((d-\lfloor\log s\rfloor-1)/d,1/2)$,
$$\max_{p\ge 1/2} \Pr_{x\in \D_{n,p}}[f(x)\not=g(x)]=\Pr_{x\in \D_{n,p'}}[f(x)\not= g(x)]\ge 2^{-H_2\left(\min\left(\frac{1}{2},\frac{\lfloor \log s\rfloor+1}{d}\right)\right)d}$$
\end{lemma}
Consider the algorithm {\bf Test} in Figure~\ref{Test_fg}. We now prove
\begin{lemma}\label{zerotest2} 
The algorithm {\bf Test}$(f,g,\delta)$ for $f,g\in \P_{n,d,s}$ given as black-boxes, makes
$$q=2^{H_2\left(\min\left(\frac{\lfloor \log s\rfloor+1}{d}\right),\frac{1}{2}\right)d}\ln\frac{1}{\delta}$$ queries, runs in time $O(qn)$, and if $f\not=g$, with probability at least $1-\delta$, returns an assignment $a$ such that $f(a)\not=g(a)$. If $f=g$ then with probability $1$ returns ``$f=g$''.
\end{lemma}
\begin{proof}
If $f\not = g$ then, by Lemma \ref{notequal} and since $1-x\le e^{-x}$, the probability that $f(a)=g(a)$ for all $a$ is at most
$$\left(1-2^{-H_2\left(\min\left(\frac{\lfloor \log s\rfloor+1}{d}\right),\frac{1}{2}\right)d} \right)^{H_2\left(\min\left(\frac{\lfloor \log s\rfloor+1}{d}\right),\frac{1}{2}\right)d\ln (1/\delta)}\le \delta. $$
\end{proof}

\begin{figure}[t]
\mybox{ {\bf Test$(f,g,\delta)$}\\ \hspace{.3in}
{\bf Input}: Black-box access to $f,g\in \P_{m,d,s}$.\\
{\bf Output}: If $f\not=g$ then find a assignment $a$ such that $f(a)\not=g(a)$.
\begin{enumerate}
    \item Let $p=\max\left(1-\frac{\lfloor \log s\rfloor+1}{d},\frac{1}{2}\right).$
    \item Repeat $2^{H_2\left(\min\left(\frac{\lfloor \log s\rfloor+1}{d},\frac{1}{2}\right)\right)d}\ln\frac{1}{\delta}$ times
    \item \hspace{.15in} Draw $a\in \D_{m,p}$ 
    \item \hspace{.15in} If $f(a)\not=g(a)$ then return $a$
    \item Return ``$f=g$''
\end{enumerate}}
\caption{For $f,g\in \P_{m,d,s}$, if $f\not=g$ then, with probability at least $1-\delta$, returns an assignment $a$ such that $f(a)\not=g(a)$.}
	\label{Test_fg}
\end{figure}

\ignore{
\begin{lemma}\label{Learnnew} 
There is an exact learning algorithm for $\P_{n,d,s}$ with
probability of success at least $1-\delta$ that runs in exponential time and makes $$q={2^{H_2\left(\min\left(\frac{\lfloor \log s\rfloor+1}{d}\right),\frac{1}{2}\right)d}}\ln\frac{|\P_{n,d,s}|}{\delta}=O\left(2^{H_2\left(\min\left(\frac{\lfloor \log s\rfloor+1}{d}\right),\frac{1}{2}\right)d}(ds\log n+\log({1}/{\delta}))\right)$$ queries.
\end{lemma}
\begin{proof}
The algorithm draws $q$ assignments $S$ according to the distribution $\D_{n,p'}$, where $p'=(d-\lfloor\log s\rfloor-1)/d$ and then finds $h\in\P_{d,s}$ that is consistent with $S$ on $f$.

By Lemma~\ref{notequal} and since $1-x\le e^{-x}$, the probability that $h\not=f$ is
\begin{eqnarray*}
\Pr [(\exists h\in \P_{d,s}, h\not=f)(\forall a\in S) h(a)=f(a)]&\le& |\P_{n,d,s}|\left(1-2^{H_2\left(\min\left(\frac{\lfloor \log s\rfloor+1}{d}\right),\frac{1}{2}\right)d}\right)^q\le \delta.
\end{eqnarray*}
\end{proof}

\section{The Learning Algorithm}\label{firstalgorithm}
In this section, we give the learning algorithm for $\P_{n,s}$. }

\subsection{The Reduction Algorithms}
In this subsection, we give the reductions we use for the learning.

Let $f(x_1,\ldots,x_n)$ be any Boolean function. A {\it $p$-zero projection} of $f$ is a random function, $f(z)=f(z_1,\ldots,z_n)$ where each $z_i$ is equal to $x_i$ with probability $p$ and is equal to $0$ with probability~$1-p$. 

We now give the first reduction,
\begin{lemma}\label{reduction1} $(\P_{n,s}\to \P_{n,d,s}).$ Let $0<p<1$ and $w=(s/\epsilon)^{\log(1/p)}\ln(16s)$. Suppose there is a proper learning algorithm that exactly learns $\P_{n,d,s}$ with $Q(d,\delta)$ queries in time $T(d,\delta)$ and probability of success at least $1-\delta$. Then there is a proper learning algorithm that learns $\P_{n,s}$ with $O(w\cdot Q(D,1/(16w))\log(1/\delta))$ queries where $$D=\log\frac{s}{\epsilon}+\frac{\log s+\log\log s+6}{\log(1/p)},$$  in time $w\cdot T(D,1/(16w))\log(1/\delta)$, probability of success at least $1-\delta$ and accuracy $1-\epsilon$.
\end{lemma}
\begin{proof}
Let $A(d,\delta)$ be a proper learning algorithm that exactly learns $\P_{n,d,s}$ with $Q(d,\delta)$ queries in time $T(d,\delta)$ and probability of success $1-\delta$. We will give an algorithm $B(d,\epsilon)$ that, with probability at least $2/3$, learns $\P_{n,s}$ with accuracy $1-\epsilon$. We will show that, with probability at least $2/3$, algorithm $B$ finds all the monomials of $f$ of size at most $\log(s/\epsilon)$. Let $h$ be the sum of those monomials. Then,
$$\Pr[h\not=f]\le s 2^{-\log(s/\epsilon)}\le \epsilon.$$
Therefore, by running $B$, $O(\log(1/\delta))$ times, we get a learning algorithm with a probability of success at least $1-\delta$.

Let $f\in \P_{n,s}$ be the target function. Algorithm $B$ chooses $w$ $p$-zero projections $f(z^{(1)}),f(z^{(2)}),\ldots$ $,f(z^{(w)})$ of $f$. 
It runs $w$ copies of $A(D,1/(16w))$ to learn all $f(z^{(i)})$, $i\in [w]$. Let $h_i$ be the hypothesis that the $i$-th copy of $A$ learns, $i\in [w]$. It then defines a hypothesis $h$ equal to the sum of all the monomials $M$ in all $h_i$ of size at most $\log(s/\epsilon)$.
The number of queries that $B$ makes is $w\cdot Q(D,1/(16w)).$ 

To prove the correctness of the algorithm, we consider the following two events
\begin{enumerate}
    \item $E_1$: For every $i$ we have $f(z^{(i)})\in \P_{n,D,s}$.
    \item $E_2$: For every monomial $M$ of $f$ of size at most $\log(s/\epsilon)$, there is $i\in[w]$ such that~$M$ is a monomial of $f(z^{(i)})$.
\end{enumerate}
If the event $E_1$ occurs, then, with probability at least $15/16$, all the copies of $A$ learn the projected functions. If $E_2$ occurs, assuming $E_1$ occurs, $h$ contains all the monomials of $f$ of size at most $\log(s/\epsilon)$.  

We have
\begin{eqnarray*}
\Pr[\mbox{not\ }E_1]&=& \Pr[(\exists i\in [w]) (\exists M\in f, |M|>D) M(z^{(i)})\not=0]\\
&\le& wsp^D\le \frac{1}{16}.
\end{eqnarray*}
and
\begin{eqnarray*}
\Pr[\mbox{not\ }E_2]&=& \Pr[(\exists M\in f, |M|\le \log(s/\epsilon))(\forall i\in [w]) M(z^{(i)})=0] \\
&\le& s\left(1-p^{\log(s/\epsilon)}\right)^w\le s\cdot exp\left(-\left(\frac{\epsilon}{s}\right)^{\log(1/p)}w\right)= \frac{1}{16}.
\end{eqnarray*}
Now assuming $E_1$ occurs, the probability that $A(D,1/(16w))$ does not learn some $f(z^{(i)})$, $i\in [w]$, is at most $1/16$. Assuming also $E_2$ occurs, $h$ is the required hypothesis.  Therefore, the probability that $h$ contains all the monomials of $f$ of size at most $\log(s/\epsilon)$ is at least $13/16>2/3$. 
\end{proof}
Before we give the second reduction, we first give two auxiliary lemmas.

The following is trivial,
\begin{lemma}\label{Trivvar}
There is a non-adaptive exact proper learning algorithm for $C=\{0,1,x_1,\ldots,x_n,\bar x_1,$ $\ldots, \bar x_n\}$ that makes $\log n+O(1)$ queries and runs in time $O(n\log n)$.
\end{lemma}

We say that $x_i$ is a {\it relevant variable} of $f$ if there is $a,b\in \{0,1\}^n$ that differ only in the $i$th coordinate such that $f(a)\not=f(b)$. We say that $f$ is {\it independent} of $x_i$ if it is irrelevant variable of~$f$.
\begin{lemma}\label{relevantf}
There is an algorithm that for any Boolean function $f\in \P_s$ that is expressed as a polynomial function and a variable $x_j$, runs in time $O(ds)$ and decides if $x_j$ is relevant variable of $f$, and if it does, finds two assignments $a$ and $b$ that differ only in the $j$th coordinate and $f(a)\not=f(b)$.  
\end{lemma}
\begin{proof}
The variable $x_j$ is a relevant variable of $f$ if it appears in one of the monomials of $f$ (in its unique polynomial expression). Now suppose $f(x)=x_jf_1(x)+f_0(x)$ where $f_0,f_1$ are independent of $x_j$. Take a monomial $M=x_{i_1}\cdots x_{i_r}$ of $f_1$ with minimal $r$. Then define $a\in\{0,1\}^n$ where $a_i=1$ if $i\in\{i_1,\ldots, i_r\}$, and $a_i=0$, otherwise. Then define $b$ to be $a$ with the $a_j$ flipped to $1$. 

Now $f_1(a)=1$, $f(a)=f_0(a)$ and $f(b)=f_1(a)+f_0(a)=1+f_0(a)\not=f_0(a)=f(a)$.
\end{proof}
We now give the second reduction. 
\begin{lemma}\label{reduction2} $(\P_{n,d,s}\to \P_{(2ds)^2,d,s}).$ Suppose there is a proper learning algorithm that exactly learns $\P_{(2ds)^2,d,s}$ with $Q(d,\delta)$ queries in time $T(d,\delta)$ and probability of success at least $1-\delta$. Then there is a proper learning algorithm that exactly learns $\P_{n,d,s}$ with $q=(Q(d,1/16)+ds\log n)\log(1/\delta)$ queries in time $(T(d,1/16)+dsn\log n)\log(1/\delta)$ and probability of success at least $1-\delta$.
\end{lemma}
\begin{proof} Let $A(d,\delta)$ be a proper learning algorithm that exactly learns $\P_{(2ds)^2,d,s}$ with $Q(d,\delta)$ queries in time $T(d,\delta)$ and probability of success at least $1-\delta$. It is enough to find a learning algorithm that exactly learns $\P_{n,d,s}$ with probability of success at least $2/3$.

Define the following algorithm $B$. 
The algorithm draws uniformly at random a map $\phi:[n]\to [m]$ where $m=(2ds)^2$, and defines $F(x_1,\ldots,x_{m})=f(x_{\phi(1)},\ldots,x_{\phi(n)})$ and then exactly learns $F$ using the algorithm $A(d,1/16)$. 
Then for every relevant variable $x_i$ of $F$, $i\in [m]$ it finds using Lemma~\ref{relevantf} two assignments $a,b\in \{0,1\}^m$ that differ only in the $i$th coordinate and $F(a)\not= F(b)$, and, using the algorithm in Lemma~\ref{Trivvar}, it learns $f(\pi^{(i,a)})$ where
$$\pi^{(i,a)}_j=\left\{\begin{array}{ll} a_{\phi(j)}& \phi(j)\not=i \\ x_j& \phi(j)=i\end{array} \right. .$$
Then algorithm $B$ returns $F(y_1,\ldots,y_m)$ where $y_i=x_j$ if $x_i$ is a relevant variable of $F$ and $f(\pi^{(i,a)})\in \{x_j,\bar x_j\}$ and $y_i=0$ otherwise. 

We now prove the correctness of the algorithm. Notice that every $f\in \P_{n,d,s}$ has at most $ds$ relevant variables.
Let $x_{r_1},\ldots,x_{r_\ell}$, $\ell\le ds$, be the relevant variables of $f$. Let $A$ be the event that $\phi(r_1),\ldots,\phi(r_\ell)$ are distinct. The probability that $A$ occurs is at least $(1-1/m)(1-2/m)\cdots(1-(\ell-1)/m)\ge 7/8$. We now assume that event $A$ occurs. In particular, $x_{\phi(r_1)},\ldots,x_{\phi(r_\ell)}$ are the relevant variables of $F$.

Let $\phi(r_i)=t_i$. Let $a,b\in \{0,1\}^m$ be two assignments that differ in the $t_i$-th coordinate and $F(a)\not=F(b)$. Assume wlog $a_{t_i}=0$ . Then $f(\pi^{(t_i,a)})$ is a non-constant function. This is because $f(\pi^{(t_i,a)})$ is a function on the variables $\{x_u\}_{\phi(u)=t_i}$ where if we substitute $0$ in all $\{x_u\}_{\phi(u)=t_i}$ we get $F(a)$ and if we substitute $1$ in all $\{x_u\}_{\phi(u)=t_i}$ we get $F(b)$.
Now since $\phi(r_i)=t_i$, and since event $A$ occurs, $x_{r_i}$ is the only variable in $\{x_u\}_{\phi(u)=t_i}$ that is relevant variable of $f$. Therefore, $f(\pi^{(t_i,a)})\in \{x_{r_i},\bar x_{r_i}\}$. 

Thus, $y_{\phi(r_i)}=y_{t_i}=x_{r_i}$ and $F(y_1,\ldots,y_m)=f(y_{\phi(1)},\ldots,y_{\phi(n)})=f$.
\end{proof}

\subsection{The Algorithm for $\P_{m,d,s}$}
In this section, we give a learning algorithm that exactly learns $\P_{m,d,s}$.

Consider the algorithm {\bf FindMonomial} in Figure~\ref{Test_f}. For two assignments $a$ and $b$ in $\{0,1\}^m$ we define $a*b=(a_1b_1,\ldots,a_mb_m)$. We prove
\begin{lemma}\label{FindMonomial} 
For $f\in \P_{m,d,s}$, $f\not=0$, {\bf FindMonomial}$(f,d,\delta)$ makes at most $$Q=O\left(2^{dH_2\left(\min\left(\frac{\lfloor \log s\rfloor+1}{d},\frac{1}{2}\right)\right)}d\log (m/\delta)\right)$$ queries, runs in time $O(Qn)$, and with probability at least $1-\delta$ returns a monomial of $f$.
\end{lemma}
\begin{proof}
Let $t=8d\ln(m/\delta)$. Let $a^{(1)},\ldots,a^{(t)}$ be the assignments generated in the ``Repeat'' loop of {\bf FindMonomial}$(f,d,\delta)$. Define the random variable $X_i=wt(a^{(i)})-d_i$, $i\in [t]$, where $d_i$ is the degree of the minimal degree monomial of $f(a^{(i)}*x)$. First notice that every monomial of $f(a^{(i)}*x)$ is a monomial of $f(a^{(i-1)}*x)$ and of $f$. Therefore, $d\ge d_{i+1}\ge d_i$. Also, if $X_i=0$ then $f(a^{(i)}*x)=\prod_{a^{(i)}_j=1}x_j$ is a monomial of $f$.

Given $a^{(i)}$ such that $g(x)=f(a^{(i)}*x)\not=0$. Notice that $g\in \P_{m,d,s}$ and if $f(b*a^{(i)}*x)=0$ then $a^{(i+1)}=a^{(i)}$, $d_{i+1}=d_i$ and $X_{i+1}=X_i$. Let $M$ be any monomial of $g$ of degree $d_i$ and suppose, wlog, $M=x_1x_2\cdots x_{d_i}$. For $b\in \D_{m,p}$ where $p=2^{-1/d}$, with probability $\eta:=(2^{-1/d})^{d_i}\ge 1/2$, $b_1=b_2=\cdots=b_{d_i}=1$. If $b_1=b_2=\cdots=b_{d_i}=1$ then $f(b*a^{(i)}*x)\not=0$. This is because $M$ remains a monomial of $f(b*a^{(i)}*x)$. Therefore, $\eta':=\Pr[f(b*a^{(i)}*x)\not=0]\ge \eta$. The expected weight of $(b_{d_i+1}a^{(i)}_{d_i+1},\ldots,b_ma^{(i)}_m)$ is $2^{-1/d}X_i$. Also, if $f(b*a^{(i)}*x)\not=0$ then, with probability at least $1/2$, {\bf Test} succeed to detect that $f(b*a^{(i)}*x)\not=0$. Therefore,
\begin{eqnarray*}\E[X_{i+1}|X_i]&\le& \frac{1}{2}\eta'(2^{-1/d}X_i)+\left(1-\frac{\eta'}{2}\right)X_i\\ &\le& \frac{1}{2}\eta'\left(1-\frac{1}{2d}\right)X_i+\left(1-\frac{\eta'}{2}\right)X_i\\
&=&\left(1-\frac{\eta'}{4d}\right)\le \left(1-\frac{1}{8d}\right)X_i.
\end{eqnarray*}
Now, $X_0\le m$ and therefore $$\E[X_t]\le m\left(1-\frac{1}{8d}\right)^t\le {\delta}.$$
Thus, by Markov's bound, the probability that $f(a^{(t)}*x)$ is not a monomial is
$$\Pr[X_t\not=0]=\Pr[X_t\ge 1]\le \delta.$$
Now, by Lemma~\ref{zerotest2}, the query complexity in the lemma follows.
\end{proof}

\begin{figure}[t]
\mybox{ {\bf FindMonomial$(f,d,\delta)$}\\ \hspace{.3in}
{\bf Input}: Black-box access to $f\in \P_{m,d,s}$\\
{\bf Output}: Find a monomial of $f$.\\
Procedure {\bf Test}$(f,g,\delta)$ in Figure~\ref{Test_fg} tests, with confidense $1-\delta$, if $f=g$ using Lemma~\ref{zerotest}. If $f\not=g$ then it returns an assignment $u$ such that $f(u)\not=g(u).$
\begin{enumerate}
    \item $a=(1,1,1,\ldots,1)\in \{0,1\}^m$.
    \item Repeat $t=8d\ln\frac{m}{\delta}$ times
    \item \hspace{.15in} Draw $b\in \D_{m,p}$ where $p=2^{-1/d}$.
    \item \hspace{.15in} {\bf Test}($f(b*a*x)$,$0$,$1/2$)
    \item \hspace{.15in} If $f(b*a*x)\not=0$ then $a\gets a*b$.
    \item Return $\prod_{a_i=1}x_i$
\end{enumerate}}
\caption{For $f\in \P_{m,d,s}$ returns a monomial of $f$.}
	\label{Test_f}
\end{figure}

We now prove
\begin{lemma}\label{Pmds} 
There is a proper learning algorithm that exactly learns $\P_{m,d,s}$, makes $$Q(m,d,\delta)=O\left(s2^{H_2\left(\min\left(\frac{\lfloor \log s\rfloor+1}{d}\right),\frac{1}{2}\right)d}d\log (ms/\delta)\right)$$ queries, and runs in  time $O(Q(m,d,\delta)n)$.

\end{lemma}
\begin{proof} 
In the first iteration of the algorithm, we run {\bf FindMonomial}$(f,d,\delta/s)$ to find one monomial. Suppose at iteration $t$ the algorithm has $t$ monomials $M_1,\ldots,M_t$ of $f$. In the $t+1$ iteration, we run {\bf FindMonomial}$(f+\sum_{i=1}^tM_i,d,\delta/s)$ to find a new monomial of $f$. 

The correctness and query complexity follows from Lemma~\ref{FindMonomial}.
\end{proof}

\subsection{The Algorithm}\label{TheAlg}
In this section, we give the algorithm for $\P_{n,s}$. We prove

\begin{theorem}\label{main01}
Let $\epsilon=1/s^\beta$, $\beta\ge 1$. There is a proper learning algorithm for $s$-sparse polynomial with probability of success at least $2/3$ that makes
\begin{eqnarray}\label{QU}
q_U=
\left(\frac{s}{\epsilon}\right)^{\gamma(\beta)+o_s(1)}+O\left(s\left(\log\frac{1}{\epsilon}\right)\log n\right)
\end{eqnarray}
queries and runs in time $O(q_U\cdot n)$
where
$$\gamma(\beta)=\min_{0\le \eta\le 1} \frac{\eta+1}{\beta+1}+(1+1/\eta)H_2\left(\frac{1}{(1+1/\eta)(\beta+1)}\right). $$

In particular
\begin{enumerate}
    \item\label{Th1} $$\gamma(\beta)=\frac{\log \beta}{\beta}+\frac{4.413}{\beta}+\Theta\left(\frac{1}{\beta^2}\right).$$
    \item \label{Th15}
    $$\begin{tabular}{|c|l|}
    \hline
        $\beta$  & $\gamma(\beta)$ \\
        \hline
        1& 2.617\\
        2 & 1.961\\
        3& 1.582\\
        4& 1.336\\
        5& 1.157\\
        6& 1.025\\
        7& 0.921\\
        \hline
    \end{tabular}$$
    \item\label{Th2} $\gamma(\beta)<1$ for $\beta>6.219$. That is, the query complexity is sublinear in $1/\epsilon$ and almost linear in $s$ when $\beta>6.219$. 
    \item\label{Th3} For $\beta>4.923$ the query complexity is better than the best known query complexity (which is $s^2/\epsilon=(s/\epsilon)^{(2+\beta)/(1+\beta)}$).
    \item\label{Th4} $\gamma(\beta)\le 4$ for all $\beta$, and $\gamma$ is a monotone decreasing function in $\beta$.
\end{enumerate}
\end{theorem}

We note here that in Section~\ref{Smallb}, we give another algorithm that improves the bounds in items~\ref{Th15}-\ref{Th4}. In particular, the query complexity of the algorithm in Section~\ref{Smallb} with the above algorithm is better than the best-known query complexity for $\beta>1$. The above algorithm also works for $\beta<1$, but the one in Section~\ref{Smallb} has a better query complexity.

We now give the proof of the Theorem

Since $\log(1/\epsilon)/(\log s)=\beta\ge 1$, we have $\log s<\log(1/\epsilon)$.
Let $p>1/2$ and $$D=\log\frac{s}{\epsilon}+\frac{\log s+\log\log s+6}{\log(1/p)}>2\log s+2.$$ We will choose $p$ later such that $\log(1/p)=\Theta((\log s)/(\log (1/\epsilon))$ and therefore $D=O(\log(s/\epsilon))=O(\log(1/\epsilon))$.

We start from the algorithm in Lemma~\ref{Pmds} that exactly learns $\P_{(2Ds)^2,D,s}$ with 
$$Q_1(D,\delta)=O\left(s2^{D\cdot H_2\left(\frac{\lfloor \log s\rfloor+1}{D}\right)}D\log ((2Ds)^2s/\delta)\right)=\tilde O\left(s2^{D\cdot H_2\left(\frac{\lfloor \log s\rfloor+1}{D}\right)}\right)\log(1/\delta)$$ queries, time $O(Q_1n)$ and probability of success at least $1-\delta$. 

By the second reduction, Lemma~\ref{reduction2}, there is a proper learning algorithm that exactly learns $\P_{n,D,s}$ with
\begin{eqnarray*}
Q_2(D,\delta)&=&(Q_1(D,1/16)+Ds\log n)\log(1/\delta)\\
&=& \left(\tilde O\left(s2^{D\cdot H_2\left(\frac{\lfloor \log s\rfloor+1}{D}\right)}\right)+ O\left(s \left(\log\frac{1}{\epsilon}\right) \log n\right)\right)\log(1/\delta).
\end{eqnarray*} queries in time $O(Q_2(D,\delta)n)$ and probability of success at least $1-\delta$.

If we now use the first reduction in Lemma~\ref{reduction1} as is we get the first summand in the query complexity in~(\ref{QU}), but the second summand, $ O(s\log(1/\epsilon)\log(n))$, becomes $(s^{1+\log(1/p)}/\epsilon^{\log(1/p)})(\log(1/\epsilon))$  $\log(n)$, which is not what we stated in the Theorem. Instead, we use the first reduction with the following changes.

Notice that the $\log n$ in the summand $ds\log(n)$ in the first reduction resulted from searching for the relevant variable in the set $\{x_u\}_{\phi(u)=t}$ for some $t\in [m]$. See the proof of Lemma~\ref{reduction2}. Suppose the algorithm knows a priori $w$ relevant variables of the function $f$ and is required to run the second reduction. Then the term $ds\log n$ can be replaced by $(ds-w)\log n$. This is because the reduction needs to search only for the other at most $ds-w$ relevant variables of $f$. Now, if we use the first reduction in Lemma~\ref{reduction1}, when we find the relevant variables of $f$ in the $p$-zero projections $f(z^{(1)}),\ldots,f(z^{(i)})$, we do not need to search for them again in the following $p$-zero projections $f(z^{(i+1)}),\ldots,f(z^{(w)})$. Therefore, the query complexity of the search of all the variables remains $ O(s\log(1/\epsilon)\log(n))$. 

Therefore, after using the first reduction with the above modification, we get a proper learning algorithm that learns $\P_{n,s}$ that makes 
$$Q_3(d,\delta,\epsilon)=\left(\tilde O\left(s\left(\frac{s}{\epsilon}\right)^{\log(1/p)}2^{D\cdot H_2\left(\frac{\lfloor \log s\rfloor+1}{D}\right)}\right)+ O\left(s \left(\log\frac{1}{\epsilon}\right) \log n\right)\right)\log(1/\delta)$$ queries in time $O(n\cdot Q_3)$, probability of success at least $1-\delta$ and accuracy $1-\epsilon$.

Now recall that $\epsilon=1/s^\beta$ and choose $p=2^{-\eta/(\beta+1)}$ for a constant $0<\eta<1$. Then $p>1/2$, $\log(1/p)=\eta/(\beta+1)=\Theta((\log s)/\log(1/\epsilon))$,
$$D=(\beta+1)\left(1+\frac{1}{\eta}\right)\log s+\Theta(\beta \log\log s)=\left(1+\frac{1}{\eta}\right)\log \frac{s}{\epsilon}+\Theta(\beta \log\log s)$$ and 
\begin{eqnarray*}
s\left(\frac{s}{\epsilon}\right)^{\log(1/p)}2^{D\cdot H_2\left(\frac{\lfloor \log s\rfloor+1}{D}\right)}&=& \left(\frac{s}{\epsilon}\right)^\frac{\eta+1}{\beta+1}\left(\frac{s}{\epsilon}\right)^{\left((1+1/\eta)+\Theta\left(\frac{\log\log s}{\log s}\right)\right)H_2\left(\frac{1}{(1+1/\eta)(\beta+1)}\left(1-\Theta\left(\frac{\log\log s}{\beta\log s}\right)\right)\right)}\\
&=& \left(\frac{s}{\epsilon}\right)^{\frac{\eta+1}{\beta+1}+(1+1/\eta)H_2\left(\frac{1}{(1+1/\eta)(\beta+1)}\right)(1-o_s(1))}.
\end{eqnarray*}
This completes the proof.

\section{Lower Bounds}\label{lowerb}
In this section, we prove the following lower bound for learning sparse polynomials.

\begin{theorem}
Let $\epsilon=1/s^\beta$. Any learning algorithm for $s$-sparse polynomial with a confidence probability of at least $2/3$ must make at least
\begin{eqnarray*}
q_L&=& \tilde \Omega\left(\left(\frac{s}{\epsilon}\right)^{\frac{\beta \cdot  H_2(\min(1/\beta,1/2))}{\beta+1}}\right)+\Omega\left(s\left(\log\frac{1}{\epsilon}\right)\log n\right)\\
    &=&\left(\frac{s}{\epsilon}\right)^{\frac{\log \beta}{\beta}+\frac{1}{(\ln 2)\beta}+\Omega(\frac{\log \beta}{\beta^2})}+\Omega\left(s\left(\log\frac{1}{\epsilon}\right)\log n\right)
\end{eqnarray*}
queries.
\end{theorem}

We first give the following lower bound that proves the second summand in the lower bound

\begin{lemma}\label{firstLB}
Any learning algorithm for $\P_{n,s}$ with a confidence probability of at least $2/3$ must make at least 
$$\Omega\left( s\left(\log\frac{1}{\epsilon}\right)\log n\right)$$ queries.
\end{lemma}
\begin{proof}
Consider the class $C=\P_{n,\log(1/(2\epsilon)),s}$. Consider a (randomized) learning algorithm $A_R$ for $\P_{n,s}$ with a confidence probability of at least $2/3$ and accuracy $\epsilon$. Then $A_R$ is also a (randomized) learning algorithm for $C$. Since by Lemma~\ref{weightofp}, any two distinct functions in $C$ have distance $2\epsilon$, $A_R$ exactly learns $C$ with a confidence probability of at least $2/3$. This is because, after learning an $\epsilon$-close formula $h$, since any two distinct functions in $C$ have distance $2\epsilon$, the closest function in $C$ to $h$ is the target function.  By Yao's minimax principle, there is a deterministic non-adaptive exact learning algorithm $A_D$ with the same query complexity as $A_R$ that learns at least $(2/3)|C|$ functions in $C$. By the standard information-theoretic lower bound, the query complexity of $A_D$ is at least $\log((2/3)|C|)$. Since
$$\log |C|=\log {{n\choose \log(1/(2\epsilon))}\choose s}= \Omega\left( \left(\log\frac{1}{\epsilon}\right)s\log n\right)$$
the result follows. 
\end{proof}

We now give the following lower bound that proves the second summand in the lower bound
\begin{lemma}\label{LB2}
Let $\epsilon=1/s^\beta$. Any learning algorithm for $\P_{n,s}$ with a confidence probability of at least $2/3$ must make at least
$$\Omega\left(\left(\frac{s}{\epsilon}\right)^{\frac{\beta \cdot H_2(\min(1/\beta,1/2))}{\beta+1}}\right)$$
queries. 
\end{lemma}
\begin{proof} We first prove the lower bound for $\beta>1$. Let $t=\log(1/\epsilon)-\log s-1$ and $r=\log s$. 
Let $W$ be the set of all pairs $(I,J)$ where  $I$ and $J$ are disjoint sets, $I\cup J=[t+r]$, $|I|\ge t$ and $|J|=t+r-|I|\le r$. For every $(I,J)\in W$ define $f_{I,J}=\prod_{i\in I}x_i\prod_{j\in J}(1+x_j).$ Consider the set $C=\{f_{I,J}|(I,J)\in W\}$. First notice that $C\subset \P_{n,t+r,s}\subseteq \P_{n,s}$ and, by Lemma~\ref{weightofp}, $\Pr[f_{I,J}=1]\ge 2^{-(t+r)}=2^{-\log(1/\epsilon)+1}=2\epsilon$. Furthermore, since for $(I_1,J_1)\not=(I_2,J_2)$ the degree of $f_{I_1,J_1}+f_{I_2,J_2}$ is $\log(1/\epsilon)-1$, we also have
\begin{eqnarray}\label{IJ}
\Pr[f_{I_1,J_1}\not=f_{I_2,J_2}]\ge 2\epsilon.
\end{eqnarray}
Therefore, any learning algorithm for $\P_{n,s}$ (with accuracy $\epsilon$ and confidence $2/3$) is a learning algorithm for $C$ and thus is an exact learning algorithm for $C$. This is because, after learning an $\epsilon$-close formula $h$, by (\ref{IJ}), the closest function in $C$ to $h$ is the target function.  

Consider now a (randomized) non-adaptive exact learning algorithm $A_R$ for $C$ with probability of success at least $2/3$ and accuracy $\epsilon$. By Yao's minimax principle, there is a deterministic non-adaptive exact learning algorithm $A_D$ such that, for uniformly at random $f\in C$, with a probability at least $2/3$, $A_D$ returns $f$. We will show that $A_D$ must make more than $q=(1/10)|C|$ queries.
Now since,
\begin{eqnarray*}
|C|&=&\sum_{i=0}^{\log s} {\log\frac{1}{\epsilon}-2\choose i}\\
&\ge&\tilde \Omega\left( 2^{H_2\left(\min\left(\frac{\log s}{\log(1/\epsilon)},\frac{1}{2}\right)\right)\log(1/\epsilon)} \right)\\
&=&\tilde\Omega\left(\left(\frac{1}{\epsilon}\right)^{H_2(\min(1/\beta,1/2))} \right)=\tilde\Omega\left(\left(\frac{s}{\epsilon}\right)^{\frac{\beta\cdot H_2(\min(1/\beta,1/2))}{\beta+1}}\right)
\end{eqnarray*}
the result follows.

To this end, suppose for the contrary, $A_D$ makes $q$ queries. Let $S=\{a^{(1)},\ldots,a^{(q)}\}$ be the queries that $A_D$ makes. For every $(I,J)\in W$ let $S_{I,J}=\{a\in S|f_{I,J}(a)=1\}$. Since for any two distinct $(I_1,J_1),(I_2,J_2)\in W$ we have $f_{I_1,J_1}\cdot f_{I_2,J_2}=0$, the sets $\{S_{I,J}\}_{(I,J)\in W}$ are disjoint sets.

Let $f=f_{I',J'}$ be uniformly at random function in $C$. We will show that, with probability at least $4/5$, $A_D$ fails to learn $f$, which gives a contradiction. 
Since
$$E_{(I,J)\in W}[|S_{I,J}|]= \frac{\sum_{(I,J)\in W}|S_{I,J}|}{|W|}=\frac{q}{w}=\frac{1}{10},$$
at least $(9/10)|W|$ of the $S_{I,J}$ are empty sets. Therefore, with probability at least $9/10$, $S_{I',J'}$ is an empty set. In other words, with probability at least $9/10$, the answers to the all the queries are~$0$. If the answers to all the queries are zero, then with probability at most $1/10$, the algorithm can guess $I',J'$, and therefore, the failure probability of the algorithm is at least $4/5$. This proves the case $\beta>1$.

Now we prove the result for $0<\beta\le 1$. By Lemma~\ref{firstLB}, we get the lower bound
$\Omega(s)$. Since
$$s=\left(\frac{s}{\epsilon}\right)^{1/(\beta+1)}$$ and for $0<\beta\le 1$
$$\frac{1}{\beta+1}\ge \frac{\beta \cdot  H_2(\min(1/\beta,1/2))}{\beta+1}$$
the result follows.
\end{proof}

\section{An Improved Algorithm for Small $\beta$}\label{Smallb}
\begin{figure}[t]
\mybox{ {\bf LearnPoly$(f,\epsilon,s)$}\\ \hspace{.3in}
{\bf Input}: Black-box access to $f\in \P_{m,s}$\\
{\bf Output}: A hypothesis $h\in \P_{m,s}$ such that  $\Pr[h=f]\ge 1-\epsilon$.
\begin{enumerate}
    \item $h\gets 0$; $\ell\gets 0$;
    \item\label{repeat} Repeat $s$ times
    \item \hspace{.16in} $t\gets 0$; find$\gets False$;
    \item \hspace{.2in}\label{while} While $t<m:=\frac{8}{7\epsilon}\ln({128s})$ and (NOT find) Do
    \item  \hspace{.4in} $t\gets t+1$;
    \item  \hspace{.4in} \label{Drawa}Draw uniformly at random $a\in\{0,1\}^m$;
    \item \hspace{.4in} \label{fhaone} If $(f+h)(a)=1$ Then 
    \item \hspace{.59in}\label{llll} $\ell\gets\ell+1;$ If $\ell=v:=(64s)\ln(128s)$ Then Output $0$ and halt;
    \item \hspace{.6in}\label{fhaone2}  $a\gets$ {\bf FindMonomial}$((f+h)(a*x),\log(s/\epsilon)+3,1/(128s))$
    \item \hspace{.6in}\label{Wta} If $wt(a)\le \log(s/\epsilon)+3$ Then 
    \item \hspace{.85in}\label{lll} $a\gets$ {\bf FindMonomial}$((f+h)(a*x),\log(s/\epsilon)+3,1/(128s))$;
    \item \hspace{.8in} find$\gets True$;
    \item \hspace{.85in}\label{putM}   $h\gets h+\prod_{a_i=1}x_i$;
    \item \hspace{.2in} If (NOT find) Then Output $h$ and halt
\end{enumerate}}
\caption{A learning algorithm for $\P_{m,s}$.}
	\label{Alg2}
\end{figure}
In this section, we prove the following. 
\begin{theorem}\label{ThIm}
Let $\epsilon=1/s^\beta$. There is a proper learning algorithm for $s$-sparse polynomial with probability of success at least $2/3$ that makes
$$q_U'=\left(\frac{s}{\epsilon}\right)^{\gamma'(\beta)}+O\left(s\left(\log\frac{1}{\epsilon}\right)\log n\right)$$
queries and runs in time $O(q_U'\cdot n)$, where 
$$\gamma'(\beta)=\max\left(1,\frac{1}{\beta+1}+H_2\left(\min\left(\frac{1}{\beta+1},\frac{1}{2}\right)\right)\right).$$

In particular,
\begin{enumerate}
    \item The query complexity of this algorithm, $q_U'$ is better than the algorithm in Theorem~\ref{main01}, $q_U$, for $\beta<6.219$.
    \item $q_U'=\tilde O(s/\epsilon)$ for $\beta\ge 3.404$
    \item The query complexity $\min(q_U,q_U')$ is better than the best-known query complexity (which is $s^2/\epsilon=(s/\epsilon)^{(2+\beta)/(1+\beta)}$) for $\beta>1$ and is equal to $O(s^2/\epsilon)$ for $0\le \beta\le 1$.
\end{enumerate}
\end{theorem}

In the following table, we compare between $\gamma'(\beta)$ in Theorem~\ref{ThIm} with $\gamma(\beta)$ in Theorem~\ref{one}.

$$\begin{tabular}{|c|l|l|}
    \hline
        $\beta$  & $\gamma(\beta)$ & $\gamma'(\beta)$ \\
        \hline
        1& 2.617&1.5\\
        2 & 1.961&1.1252\\
        3& 1.582&1.1061\\
        4& 1.336&1\\
        5& 1.157&1\\
        6& 1.025&1\\
        7& 0.921&1\\
        8& 0.839&1\\
         9& 0.77&1\\
          10& 0.713&1\\
        \hline
    \end{tabular}$$

We first prove
\begin{lemma}\label{Iom}
Let $\epsilon=1/s^\beta$. There is a proper learning algorithm for $\P_{m,s}$ with probability of success at least $15/16$, makes
$$Q=\tilde O\left(\frac{s}{\epsilon}\right)+  \left(\frac{s}{\epsilon}\right)^{\frac{1}{\beta+1}+H_2\left(\min\left(\frac{1}{\beta+1},\frac{1}{2} \right)\right)(1+o_s(1))}\log m$$
queries, and runs in time $O(Q\cdot m)$.
\end{lemma}
\begin{proof} 
Consider the algorithm {\bf LearnPoly} in Figure~\ref{Alg2}. The algorithm uses the procedure {\bf FindMonomial} $(f,d,\delta)$ that, for $f\in\P_{n,d,s}$, with probability at least $1-\delta$, returns an assignment $a$ such that $f(a*x)=\Pi_{a_i=1}x_i$ is a monomial of $f$. If $f\not\in \P_{n,d,s}$ and $f\not=0$ then it returns an assignment $a$ such that $f(a*x)\not=0$, but, $\Pi_{a_i=1}x_i$ is not necessarily a monomial of $f$. Also, if $a$ satisfies a monomial of $f(x)$ of size at most $\log(s/\epsilon)+3$ and does not satisfy any monomial of size more than $\log(s/\epsilon)+3$ then {\bf FindMonomial}$(f(a*x),\log(s/\epsilon)+3,\delta)$, with probability at least $1-\delta$, returns an assignment that corresponds to a  monomial of size at most $\log(s/\epsilon)+3$. This is because, for such an $a$, $f(a*x)$ contains no monomials of size greater than $\log(s/\epsilon)+3$. See the algorithm {\bf FindMonomial} in Figure~\ref{Test_f} and Lemma~\ref{FindMonomial}.

The Repeat-loop in step~\ref{repeat} is executed $s$ times, and at each iteration, the algorithm, whp, either adds to $h$ a monomial of $f$ of size at most $\log(s/\epsilon)+3$ that is not in $h$ or detects that $\Pr[(f+h)(x)\not=1]=\Pr[f(x)\not= h(x)]\le \epsilon$. In the While-loop in step~\ref{while}, $h$, whp, contains some monomials of $f$. The algorithm  searches for an assignment $a$ that satisfies $f+h$, i.e., an assignment that satisfies a monomial of $f$ that is not in $h$. If such an assignment is found, step~\ref{fhaone2} uses the procedure
{\bf FindMonomial} to, whp, finds a new monomial of $f$ of size at most $\log(s/\epsilon)+3$. This procedure runs again in step~\ref{lll}. If such a monomial is found, it is added to $h$ in step~\ref{putM}.  

We call {\bf FindMonomial} twice because when in step~\ref{fhaone} a positive assignment $a$ is found for $f+h$, it may happen that $a$ satisfies some monomials of size more than $\log(s/\epsilon)+3$. In that case, we cannot guarantee that the first call returns a monomial of $f+h$ of size at most $\log(s/\epsilon)+3$. See the first paragraph in this proof.

Now, the algorithm may fail if one of the following events occurs.
\begin{enumerate}
\item $A_1$: In at least one of the (at most $s$) executions of {\bf FindMonomial} in step~\ref{lll}, the procedure fails to return an assignment $a$ such that $\prod_{a_i=1}x_i$ is a monomial of $f$ of size at most $\log (s/\epsilon)+3$.
\item $A_2$: In one of (the at most $s$) iterations in the Repeat-loop we have $\Pr[f(x)\not=h(x)]>\epsilon$ and for $m=(8/(7\epsilon))\ln(128s)$ assignments $a$ drawn uniformly at random, no one satisfies a monomial of $F(x):=f(x)+h(x)$ of size at most $\log(s/\epsilon)+3$ and does not satisfy any monomial of size more than $\log(s/\epsilon)+3$. 

Notice that when $A_2$ occurs, the variable ``find'' remains ``False'', and the algorithm returns $h$ that is $\epsilon$-far from $f$. 

\item $A_3$: For one of the assignments $a$ that satisfies a monomial of $F(x):=f(x)+h(x)$ of size at most $\log(s/\epsilon)+3$ and does not satisfy any monomial of $F$ of size more than $\log(s/\epsilon)+3$, {\bf FindMonomial} in step~\ref{fhaone2} fails to output an assignment of weight at most $\log(s/\epsilon)+3$. 
\item $A_4$: The command {\bf FindMonomial} in step~\ref{fhaone2} runs more than $v$ times. Notice that the variable $\ell$ counts the number of times that the command {\bf FindMonomial} in step~\ref{fhaone2} runs. If $\ell=v$ then the algorithm outputs an arbitrary $h=0$.
\end{enumerate} 

Consider the event $A_1$. Let $F=f+h$. Consider steps~\ref{Wta}-\ref{putM}. The assignment $a$ satisfies $F(a*x)\not=0$. This follows from steps~\ref{fhaone} and~\ref{fhaone2}. Also, by step~\ref{Wta}, $wt(a)\le \log(s/\epsilon)+3$ and therefore $F(a*x)$ has at most $\log(s/\epsilon)+3$ relevant variables and all its monomials are of size at most $\log(s/\epsilon)+3$. By Lemma~\ref{FindMonomial}, with probability at least $1-1/(128s)$, {\bf FindMonomial}$(F(a*x),\log(s/\epsilon)+3,1/(128s))$ returns an assignment $b$ such that $M=\prod_{b_i=1}x_i$ is a monomial of $F(a*x)$. Since the monomials of $F(a*x)$ are monomials of $F(x)$, the claim follows. This is for one call to {\bf FindMonomial}. Since each time we call {\bf FindMonomial}, with probability at least $1-1/(128s)$, we find a new monomial of $f$, and since this command runs at most $s$ time, we have $\Pr[A_1]\le 1/128$.

Consider the events $A_2$ and $A_3$. In step~\ref{Drawa}, the algorithm draws an assignment $a$ uniformly at random. The probability that $F(a)=1$ is at least $\epsilon$. The probability that $a$ satisfies at least one of the monomials in $F$ of size more than $\log(s/\epsilon)+3$ is at most $s2^{-\log(s/\epsilon)+3}\le \epsilon/8$. Therefore, the probability that $F(a)=1$ and $a$ does not satisfy any one of the monomials of size more than $\log(s/\epsilon)+3$ is at least $7\epsilon/8$. Thus, with probability at least $7\epsilon/8$, $F(a)=1$ and $F(a*x)$ is of degree $\log(s/\epsilon)+3$ polynomial. The probability that the algorithm fails to find such an $a$ in the While-loop is
$$(1-7\epsilon/8)^{\frac{8}{7\epsilon}\ln{128s}}\le\frac{1}{128s}.$$
Once such an $a$ is found, by Lemma~\ref{FindMonomial}, with probability at least $1-1/(128s)$, {\bf FindMonomial}$(F(a*x),\log(s/\epsilon)+3,1/(128s))$ in step~\ref{fhaone2} returns an assignment $a$ of weight at most $\log(s/\epsilon)+3$. In that case, the algorithm adds a monomial to $h$, finishes the While-loop, and returns to the Repeat-loop. Therefore, the probability of $A_2\vee A_3$ is at most $1/64$.

Before we consider $A_4$, we will first assume that the events $\bar A_1, \bar A_2$, and $\bar A_3$ occur and find an upper bound for the expected number of times that step~\ref{fhaone2} is executed.

Let $N$ be a random variable representing the number of times that step~\ref{fhaone2} is executed. Let $N_1$ and $N_2$ be the number of times it is executed while $\Pr[F(x)=1]\ge \epsilon$ and, $\Pr[F(x)=1]< \epsilon$, respectively. Obviously, $\E[N]=\E[N_1]+\E[N_2]$. We will now upper bound $\E[N_1]$ and $\E[N_2]$.

Consider the case when $\Pr[F(x)=1]\ge \epsilon$. It is well known (and easy to prove) that
$$\E[N_1]=\sum_{\nu=1}^\infty \Pr[N_1\ge \nu].$$
Let $X_i$, $i\ge 1$, be an indicator random variable that is equal to $0$ if, at the $i$th time that steps~\ref{llll}-\ref{Wta} are executed  (i.e., $i$th time that $F(a)=1$ in step~\ref{fhaone}), both commands in steps~\ref{fhaone2} and~\ref{lll} are executed and $1$ if only the command in step~\ref{fhaone2} is executed. Let $W(a)$ be the event that, for a uniform at random assignment $a$, no monomial of $f$ of size more than $\log(s/\epsilon)+3$ satisfies $a$.
If $F(a)=1$ and $W(a)$ occurs then, assuming $\bar A_1, \bar A_2$ and $\bar A_3$ occur, the command in step~\ref{fhaone2} outputs an assignment of weight at most $\log(s/\epsilon)+3$ and then the command in step~\ref{lll} is also executed. Therefore,
\begin{eqnarray*}
\E[X_i]&\le& 1- \Pr[W(a)|F(a)=1]\\
&\le& \Pr[\neg W(a)|F(a)=1]\\
&\le & \frac{\Pr[\neg W(a)]}{\Pr[F(a)=1]}\\
&\le & \frac{s2^{-(\log(s/\epsilon)+3)}}{\epsilon}=\frac{1}{8}.\\
\end{eqnarray*}

Notice that this bound is independent of $X_{i-1},\ldots,X_1$. That is, $\Pr[X_i=1|X_{i-1},\ldots,X_1]\le 1/8$.
Since the command in step~\ref{lll} is executed at most $s$ times, the event $N_1\ge \nu$ implies the event $X_1+\cdots+X_\nu\ge \nu-s$. By the generalized Chernoff's bound,~\cite{PanconesiS97,ImpagliazzoK10}, for $\nu\ge 4s$, we have
$$\Pr[N_1\ge \nu]\le \Pr\left[X_1+\cdots+X_\nu\ge \frac{\nu-s}{\nu}\nu\right]\le e^{-\nu D((\nu-s)/\nu\|1/8)}\le e^{-2\nu((\nu-s)/\nu-1/8)^2}\le e^{-\nu/2}.$$
Therefore 
$$\E[N_1]=\sum_{\nu=1}^\infty \Pr[N_1\ge \nu]\le 4s+\sum_{\nu=4s+1}^\infty e^{-\nu/2}\le 4s+1.$$

To upper bound $\E[N_2]$, consider now the case when $\Pr[F(x)=1]<\epsilon$. Since the Repeat-loop with the While-loop runs at most $r=(8s/(7\epsilon))\log(128s)$ times and in each iteration the probability that step \ref{fhaone2} is executed is $\Pr[F(x)=1]<\epsilon$, the expected number of times the algorithm executes the command in step \ref{fhaone2} is at most
$$\E[N_2]=\epsilon r=\frac{8s}{7}\log ({128s}).$$
Therefore,
$$\E[N]\le \frac{8s}{7}\log({128s})+4s+1.$$
Therefore, By Markov's bound
$$\Pr[A_4]=\Pr\left[N\ge 64s\ln(128s)\right]\le \frac{1}{32}.$$

Now
\begin{eqnarray*}
\Pr[A_1\vee A_2\vee A_3\vee A_4]&\le & \Pr(A_1)+\Pr(A_2|\bar A_1)+\Pr(A_3|\bar A_1\wedge \bar A_2)+\Pr(A_4|\bar A_1\wedge \bar A_2\wedge \bar A_3)\\
&\le& 1/16.
\end{eqnarray*}
This completes the correctness of the algorithm.

Now for the query complexity of the algorithm, we have the following.
\begin{enumerate}
    \item\label{QC01} The query complexity of step~\ref{fhaone} is $\tilde O(s/\epsilon)$.
    
    This is because we have two loops in the algorithm. The Repeat-loop runs at most $s$ iterations, and the While-loop at most $(8/(7\epsilon))\ln(128s)$ iterations. 
    \item\label{QC02} The query complexity of step~\ref{lll} is $$q:=\left( \left(\frac{s}{\epsilon}\right)^{\frac{1}{\beta+1}+H_2\left(\min\left(\frac{1}{\beta+1},\frac{1}{2} \right)\right)(1+o_s(1))}\log m\right)$$
    
This follows from the fact that  step~\ref{lll} is executed at most $s$ times. 
By Lemma~\ref{FindMonomial}, for $d=\log(s/\epsilon)+3$, the query complexity of this step is
$$q=O\left(s2^{dH_2\left(\min\left(\frac{\lfloor \log s\rfloor+3}{d},\frac{1}{2}\right)\right)}d\log m\right)=\left( \left(\frac{s}{\epsilon}\right)^{\frac{1}{\beta+1}+H_2\left(\min\left(\frac{1}{\beta+1},\frac{1}{2} \right)\right)(1+o_s(1))}\log m\right).$$
    \item\label{QC03} The query complexity of step~\ref{fhaone2} is $v=\tilde O(q)$.
\end{enumerate}
This completes the proof. 
\end{proof}

%\begin{lemma}\label{reduA}$(\P_{n,s}\to \P_{m,s}).$ Suppose there is a proper learning algorithm that learns $\P_{m,s}$ with $ Q(m,\delta)$ queries in time $T(m,\delta)$, probability of success at least $1-\delta$ and accuracy $1-\epsilon$. Then there is a proper learning algorithm that learns $\P_{n,s}$ with $O(Q(m,1/16)\log(1/\delta)+O(s(\log(1/\epsilon))\log n)$ queries in time $ O(T(m,1/16)+Q()n)\log(1/\delta)$, probability of success at least $1-\delta$ and accuracy $1-\epsilon$.
%\end{lemma}

We are now ready to prove Theorem~\ref{ThIm}.
\begin{proof}
%Let $A(d,\delta,\epsilon)$ be a proper learning algorithm that learns $\P_{m,s}$ with $ Q(m,\delta)$ queries in time $T(m,\delta)$, probability of success at least $1-\delta$ and accuracy $1-\epsilon$.
 The algorithm first takes a $(1-1/(64s\log(2s/\epsilon)))$-zero projection $g=f(z_1,\ldots,z_n)$ of $f$. The probability that some monomial of $f$ of size at most $\log(2s/\epsilon)$ is not of $g$ is less than
$$1-\left(1-\frac{1}{64s\log(2s/\epsilon)}\right)^{s\log(2s/\epsilon)}\le \frac{1}{64}.$$
The probability that $g$ has a monomial of size $d:=64s\log(2s/\epsilon)\ln (64s)$ is at most
$$s\left(1-\frac{1}{64s\log(2s/\epsilon)}\right)^{64s\log(2s/\epsilon)\ln (64s)}\le se^{-\ln(64s)}\le \frac{1}{64}.$$
Therefore, with probability at least $31/32$, $g$ is a degree-$d$ $s$-sparse polynomial and
\begin{eqnarray}\label{fineq}
\Pr[f(x)\not=g(x)]\le se^{-\log(2s/\epsilon)}= \frac{\epsilon}{2}.
\end{eqnarray}
In particular, with probability at least $31/32$, $g$ contains at most $ds=64s^2\log(2s/\epsilon) \ln(64s)$ relevant variables.

The algorithm then continues as in Lemma~\ref{reduction2}. It draws uniformly at random a map $\phi:[n]\to [m]$ where $m=16(ds)^2$, and defines $G(x_1,\ldots,x_{m})=g(x_{\phi(1)},\ldots,x_{\phi(n)})$ and then learns $G$ using the algorithm in Lemma~\ref{Iom} with accuracy $\epsilon/2$. The probability that different relevant variables of $f$ are mapped by $\phi$ into different variables of $G$ is 
\begin{eqnarray}\label{klk}
1-\frac{(ds)^2}{2m}=\frac{31}{32}.
\end{eqnarray}

The algorithm in Lemma~\ref{Iom} learns a hypothesis $H(x_1,\ldots,x_m)$ that contains some of the monomials of size at most $\log(2s/\epsilon)+3$ of $G$ and, with probability at least $15/16$, $\Pr[H\not=G]\le \epsilon/2$. 

Then for every relevant variable $x_i$ of $H$, $i\in [m]$ it finds two assignments $a,b\in \{0,1\}^m$ that differ only in the $i$th coordinate and $G(a)\not= G(b)$. Here, we cannot use Lemma~\ref{relevantf} (as we did in Lemma~\ref{reduction2}) because we have not learned $G$ but $H$ that is $\epsilon/2$-close to $G$. To find such assignments, we take any monomial $M=x_{i_1}\cdots x_{i_r}$ of $H$ that contains $x_i$. Let $G_M$ be $G$ where we substitute zero in every variable $x_j$, $j\not\in\{i_1,\ldots,i_r\}$. Since the monomials of $H$ are also monomials of $G$, then $G_M\not=0$ and $x_i$ is a relevant variable of $G_M$. Also, since $H$ contains only monomials of degree $\log(2s/\epsilon)+3$, $G_M$ does too. Then to get the two assignments, we run {\bf Test}$(G_M,G_M',1/(32m))$ where $G_M'$ is $G_M$ with the substitution of $x_i=0$.

Then using the algorithm in Lemma~\ref{Trivvar}, we learn $g(\pi^{(i,a)})$ where
$$\pi^{(i,a)}_j=\left\{\begin{array}{ll} a_{\phi(j)}& \phi(j)\not=i \\ x_j& \phi(j)=i\end{array} \right. .$$
Then we return $H(y_1,\ldots,y_m)$ where $y_i=x_j$ if $x_i$ is a relevant variable of $G$, and $g(\pi^{(i,a)})\in \{x_j,\bar x_j\}$ and $y_i=0$ otherwise. We now claim that, with probability at least $15/16$,
\begin{eqnarray}\label{sineq}
\Pr[H(y_1,\ldots,y_m)=g(x)]\le \epsilon/2.
\end{eqnarray}
The proof of this claim is the same as Lemma~\ref{reduction2}.

Now, by (\ref{fineq}), (\ref{klk}), (\ref{sineq}), and the probability of success at least $31/32$ of the {\bf Test} calls, with probability at least $1-(1/16+1/32+1/32+1/32)>2/3$ we have  $\Pr[H(y_1,\ldots,y_m)=f(x)]\le \epsilon.$

Now for the query complexity. By Lemma~\ref{Iom}, the query complexity of learning $G$ is 
\begin{eqnarray}\label{www1}
Q=\tilde O\left(\frac{s}{\epsilon}\right)+  \left(\frac{s}{\epsilon}\right)^{\frac{1}{\beta+1}+H_2\left(\min\left(\frac{1}{\beta+1},\frac{1}{2} \right)\right)(1+o_s(1))}\log m=\tilde O\left(\frac{s}{\epsilon}\right)+  \left(\frac{s}{\epsilon}\right)^{\frac{1}{\beta+1}+H_2\left(\min\left(\frac{1}{\beta+1},\frac{1}{2} \right)\right)(1+o_s(1))}.
\end{eqnarray}
We now find the query complexity of finding the relevant variables. The number of relevant variables of $H$ is at most $v=(\log(2s/\epsilon)+3)s$. For each variable we run {\bf Test} that by, Lemma~\ref{zerotest2}, takes 
$$q=2^{H_2\left(\min\left(\frac{\lfloor \log s\rfloor+1}{\log(2s/\epsilon)+3}\right),\frac{1}{2}\right)(\log(2s/\epsilon)+3)}\ln{(32m)}=\left(\frac{s}{\epsilon}\right)^{H_2\left(\min\left(\frac{1}{\beta+1},\frac{1}{2} \right)\right)(1+o_s(1))}$$ queries. So the total number of queries for all the calls is
\begin{eqnarray}\label{www2}
qv=\left(\frac{s}{\epsilon}\right)^{\frac{1}{\beta+1}+H_2\left(\min\left(\frac{1}{\beta+1},\frac{1}{2} \right)\right)(1+o_s(1))}.
\end{eqnarray}
Then the query of the searches is
\begin{eqnarray}\label{www3}
v\log n= \tilde O(s)\left(\log\frac{1}{\epsilon}\right)\log n.
\end{eqnarray}
By (\ref{www1}), (\ref{www2}), and (\ref{www3}), the result follows.
\end{proof}

\bibliography{TestingRef}

\end{document}